\documentclass[runningheads,a4paper]{llncs}
\usepackage{amsmath} \usepackage{amstext} \usepackage{amsbsy}
\usepackage{prettyref} \usepackage{latexsym} \usepackage{amssymb}
\usepackage{graphicx} \usepackage{hyperref} \usepackage{braket}
\usepackage{aas_macros}

\newrefformat{alg}{Algorithm~\ref{#1}}
\newrefformat{eq}{Equation~\ref{#1}}
\newrefformat{lem}{Lemma~\ref{#1}}
\newrefformat{def}{Definition~\ref{#1}}
\newrefformat{cor}{Corollary~\ref{#1}}
\newrefformat{chp}{Chapter~\ref{#1}}
\newrefformat{sec}{Section~\ref{#1}}
\newrefformat{apx}{Appendix~\ref{#1}}
\newrefformat{tab}{Table~\ref{#1}} \newrefformat{fig}{Fig.~\ref{#1}}

 \newcommand{\fancy}[1]{\mathcal{#1}}

\titlerunning{Measuring Intelligence} \title{Ultimate Intelligence
  Part III: Measures of Intelligence, Perception and Intelligent
  Agents}

\author{Eray \"Ozkural}

\date{\today}

\begin{document}

\mainmatter % start of an individual contribution

% \maketitle

% the affiliations are given next; don't give your e-mail address
% unless you accept that it will be published
\institute{G\"{o}k Us Sibernetik Ar\&Ge Ltd. \c{S}ti.\\
}

\maketitle

\begin{abstract}
  We propose that operator induction serves as an adequate model of
  perception.  We explain how to reduce universal agent models to
  operator induction.  We propose a universal measure of operator
  induction fitness, and show how it can be used in a reinforcement
  learning model and a homeostasis (self-preserving) agent based on
  the free energy principle.  We show that the action of the
  homeostasis agent can be explained by the operator induction model.
\end{abstract}

``Wir m\"{u}ssen wissen -- wir werden wissen!''
\begin{flushright}--- David Hilbert \end{flushright}

\section{Introduction}

The ultimate intelligence research program is inspired by Seth Lloyd's
work on the ultimate physical limits to computation
\cite{seth-ultimate}. We investigate the ultimate physical limits and
conditions of intelligence. This is the third installation of the
paper series, the first two parts proposed new physical complexity
measures, priors and limits of inductive inference
\cite{ozkural-agi15,ozkuraluip2arxiv}.

We frame the question of ultimate limits of intelligence in a general
physical setting, for this we provide a general definition of an
intelligent system and a physical performance criterion, which as
anticipated turns out to be a relation of physical quantities and
information, the latter of which we had conceptually reduced to
physics with minimum machine volume complexity in
\cite{ozkural-agi15}.

\section{Notation and Background}

\subsection{Universal Induction}
 
Let us recall Solomonoff's universal distribution \cite{alp1}. Let $U$
be a universal computer which runs programs with a prefix-free
encoding like LISP; $y=U(x)$ denotes that the output of program $x$ on
$U$ is $y$ where $x$ and $y$ are bit strings. \footnote{A prefix-free
  code is a set of codes in which no code is a prefix of another. A
  computer file uses a prefix-free code, ending with an EOF symbol,
  thus, most reasonable programming languages are prefix-free.
  % Prefix-free condition is required to satisfy Kraft inequality.
}Any unspecified variable or function is assumed to be represented as
a bit string. $|x|$ denotes the length of a bit-string $x$.
$f(\cdot)$ refers to function $f$ rather than its application.

The algorithmic probability that a bit string $x \in \{0,1\}^+$ is
generated by a random program $\pi \in \{0,1\}^+$ of $U$ is:
\begin{equation}
  \label{eq:alp}
  P_U(x) = \sum_{U(\pi) \in x(0|1)^* \wedge \pi \in  \{0,1\}^+} 2^{-|\pi|}
\end{equation}
which conforms to Kolmogorov's axioms \cite{levin-thesis}.  $P_U(x)$
considers any continuation of $x$, taking into account non-terminating
programs.\footnote{We used the regular expression notation in language
  theory.}  $P_U$ is also called the universal prior for it may be
used as the prior in Bayesian inference, for any data can be encoded
as a bit string.
% We shall denote it by merely $P$ in the rest of the paper, when we
% can discern probability of bit strings from probability measures.
We also give the basic definitions of Algorithmic Information Theory
(AIT) \cite{Li2008}, where the algorithmic entropy, or complexity of a
bit string $x \in \{0,1\}^+$ is
\begin{align}
  \label{eq:algo-entropy}
  H_U(x) &= \min( \{ |\pi| \ | \  U(\pi)=x \} )
  & H^*_U(x) &= -\log_2P_U(x)
\end{align}

We use some variables in overloaded fashion in the paper, e.g., $\pi$
might be a program, a policy, or a physical mechanism depending on the
context.

\subsection{Operator induction}

Operator induction is a general form of supervised machine learning
where we learn a stochastic map from $n$ question and answer pairs
$D=\{ (q_i, a_i) \}$ sampled from a (computable) stochastic source
$\mu$.  Operator induction can be solved by finding in available time
a set of operators $O^j(\cdot|\cdot)$, each a conditional probability
density function (cpdf), such that the following goodness of fit is
maximized
\begin{equation}
  \label{eq:opind-gof}
  \Psi = \sum_j{\psi^j_n}
\end{equation}
for a stochastic source $\mu$ where each term in the summation is the
contribution of a model:
\begin{equation}
  \label{eq:opind-gof-term}
  \psi^j_n= 2^{-|(O^j(\cdot|\cdot)|}\prod_{i=1}^n{O^j(a_i|q_i)}.
\end{equation}
$q_i$ and $a_i$ are question/answer pairs in the input dataset drawn
from $\mu$, and $O^j$ is a computable cpdf in
\prettyref{eq:opind-gof-term}.  We can use the found $m$ operators to
predict unseen data with a mixture model \cite{solomonoff-threekinds}
\begin{equation}
  \label{eq:opind-pred}
  P_U(a_{n+1}|q_{n+1}) = \sum_{j=1}^m\psi^j_nO^j(a_{n+1}|q_{n+1})
\end{equation}
The goodness of fit in this case strikes a balance between high a
priori probability and reproduction of data like in minimum message
length (MML) method \cite{WallaceBoulton:1968,wallace99}, yet uses a
universal mixture like in sequence induction.  The convergence theorem
for operator induction was proven in \cite{solomonoff-progress} using
Hutter's extension to arbitrary alphabet, and it bounds total error by
$H_U(\mu)\ln 2$ similarly to sequence induction.
% \begin{equation}
%   \label{eq:opindconv}
%   \sum_l\mu(Z_l)
%   \sum_{i=1}^n\sum_{j=0}{h^l_i+1}\sum_{t=0,1,s}(P^l_{i,j}(t)-\mu^l_{i,j}(t))^2
%   < k \ln 2
% \end{equation}

\subsection{Set induction}

Set induction generalizes unsupervised machine learning where we learn
a probability density function (pdf) from a set of $n$ bitstrings
$D = \{ d_1, d_2, ..., d_n \}$ sampled from a stochastic source $\mu$.
We can then inductively infer new members to be added to the set with:
\begin{equation}
  \label{eq:2}
  P(d_{n+1}) = \frac{P_U( D \cup d_{n+1}) }{P_U(D)}
\end{equation}
Set induction is clearly a restricted case of operator induction where
we set $Q_i$'s to null string. Set induction is a universal form of
clustering, and it perfectly models perception. If we apply set
induction over a large set of 2D pictures of a room, it will give us a
3D representation of it necessarily. If we apply it to physical sensor
data, it will infer the physical theory -- perfectly general, with
infinite domains -- that explains the data, perception is merely a
specific case of scientific theory inference in this case, though set
induction works both with deterministic and non-deterministic
problems.

\subsection{Universal measures of intelligence}

There is much literature on the subject of defining a measure of
intelligence. Hutter has defined an intelligence order relation in the
context of his universal reinforcement learning (RL) model AIXI
\cite{hutter-aixigentle}, which suggests that intelligence corresponds
to the set of problems an agent \textit{can} solve. Also notable is
the universal intelligence measure \cite{legg-uai,legg-aiq}, which is
again based on the AIXI model. Their universal intelligence measure is
based on the following philosophical definition compiled from their
review of definitions of intelligence in the AI literature.
\begin{definition}[Legg \& Hutter]
  \label{def:legg-uai}
  Intelligence measures an agent's ability to achieve goals in a wide
  range of environments.
\end{definition}
It implies that intelligence requires an autonomous goal-following
agent. The intelligence measure of \cite{legg-uai} is defined as
\begin{equation}
  \label{eq:legg-uai}
  \Upsilon(\pi) = \sum_{\mu \in E} 2^{-H_U(\mu)}V_\mu^\pi
\end{equation}
where $\mu$ is a computable reward bounded environment, And
$V_\mu^\pi$ is the expected sum of future rewards in the total
interaction sequence of agent $\pi$.
$V^\pi_\mu = E_{\mu, \pi}\left[\sum_{t=1}^\infty \gamma^t r_t\right]$,
where $r_t$ is the instantaneous reward at time $t$ generated from the
interaction between the agent $\pi$ and the environment $\mu$, and
$\gamma^t$ is the time discount factor.

\subsection{The free energy principle.}

In Asimov's story titled ``The Last Question'', the task of life is
identified as overcoming the second law of thermodynamics, however futile.
Variational free energy essentially measures predictive error, and it
was introduced by Feynmann to address difficult path integral problems
in quantum physics. In thermodynamic free energy, energies are
negative log probabilities like entropy. The free energy principle
states that any system must minimize its free energy to maintain its
order.  An adaptive system that tends to minimize average surprise
(entropy) will tend to survive longer.  A biological organism can be
modelled as an adaptive system that has an implicit probabilistic
model of the environment, and the variational free energy puts an
upper bound on the surprise, thus minimizing free energy will improve
the chances of survival.  The divergence between the pdf of
environment and an arbitrary pdf encoded by its own mechanism is
minimized in Friston's model \cite{friston-free-energy}.
% The free energy principle models cognition by the tendency of an
% agent to minimize its free energy.
It has been shown in detail that the free energy principle adequately
models a self-preserving agent in a stochastic dynamical system
\cite{Friston2006,friston-free-energy}, which we can interpret as an
environment with computable pdf. An active agent may be defined in the
formalism of stochastic dynamical systems, by partitioning the
physical states $X$ of the environment into
$X = E \times S \times A \times \Lambda$ where $e \in E$ is an
external state, $s \in S$ is a sensory state, $a \in A$ an active
state, and $\lambda \in \Lambda$ is an internal
state. Self-preservation is defined by the Markov blanket
$S \times A$, the removal of which partitions $X$ into external states
$E$ and internal states $\Lambda$ that influence each other only
through sensory and action states. $E$ influences sensations $S$,
which in turn influence internal states $\Lambda$, resulting in the
choice of action signals $S$, which impact $E$, forming the feedback
loop of the adaptive system. The system states $\vec{x} \in X$ evolve
according to the stochastic equation: \vspace{-33pt}
\begin{multicols}{2}
  \begin{align}
    \dot{\vec{x}}(t) &= f(\vec{x}) + \omega \\
    {\vec{x}(0)} &= \vec{x_0}
  \end{align}
  \break
  \begin{align}
    f(\vec{x}) = \begin{bmatrix}
      f_e(e,s,a)           \\[0.3em]
      f_s(e,s,a)           \\[0.3em]
      f_a(s,a,\lambda)     \\[0.3em]
      f_\lambda(s,a,\lambda)
    \end{bmatrix}
  \end{align}
\end{multicols}
where $f(\vec{x})$ is the flow of system states and it is decomposed
into flows over the sets in the system partition, explicitly showing
the dependencies among state sets; $\omega$ models
fluctuations. Friston formalizes the self-preservation (homeostasis)
problem as finding an internal dynamics that minimizes the uncertainty
(Shannon entropy) of the external states, and shows a solution based
on the principle of least action \cite{friston-free-energy} wherein
minimizing free energy is synonymous with minimizing the entropy of
the external states (principle of least action), which subsequently
corresponds to active inference.  We have space for only some key
results from the rather involved mathematical theory.  $p(s,f|m) $ is
the generative pdf that generates sensorium $s$ and fictive (hidden)
states $f \in F$ from probabilistic model $m$, and $q(f|\lambda)$ is
the recognition pdf that predicts hidden states $F$ in the world given
internal state.  Generative pdf factorizes as
$p(s,f | m)=p(s | f,m)p(f | m)$.
% Surprise is defined as the entropy of the sensorium and hidden
% states
Free energy is defined as energy minus entropy
\begin{equation}
  \label{eq:free-energy}
  F(s,\lambda)  =  E_q[-\ln{p(s,f|m)}] - H(q(f|\lambda))
\end{equation}
which can be subjectively computed by the system.  Free energy is also
equal to surprise plus divergence between recognition and generative
pdf's.
\begin{equation}
  \label{eq:free-energy2}
  F(s,\lambda)  = E_q[-\ln{p(s,f|m)}] + D_{KL}({q(f|\lambda)} ||  {p(f|s,m)})
\end{equation}
Minimizing divergence minimizes free energy, internal states $\lambda$
may be optimized to minimize predictive error using
\prettyref{eq:free-energy2}, and surprise is invariant with respect
to $\lambda$. Free energy may be formulated as complexity plus 
accuracy of recognition, as well.
\begin{equation}
  \label{eq:free-energy3}
  F(s,\lambda)  = E_q[-\ln{p(s,a|f,m)}] + D_{KL}({q(f|\lambda)} ||  {p(f,m)})
\end{equation}
In this case, we may choose an action that changes sensations to
reduce predictive error. Only the first term is a function of action signals. 
Minimization of free energy turns out to be
equivalent to the information bottleneck principle of Tishby
\cite{friston-free-energy,tishby-infobottleneck}. The information
bottleneck method is equivalent to the pioneering work of Ashby, which
is simple enough to state here
\cite{ashby1947principles,ashby1962principles}:
\begin{equation}
  \label{eq:info-bottleneck}
  \fancy{S}_B = I(\lambda; F) - I(S; \lambda)
\end{equation}
where the first term is the mutual information between internal and
hidden states, and the second term is the mutual information between
sensory states and internal states. Both terms are expanded using
conditional entropy, and then two terms in the middle are eliminated
because they are not relevant to the optimization problem -- we do not
know the hidden variables in $H(\lambda | F)$ and $H(S)$ is constant.
\begin{align}
  \fancy{S}_B &= H(\lambda) - H(\lambda | F) - H(S) + H(S | \lambda) \\
  \fancy{S}^*_B &= H(\lambda) + H(S | \lambda)
                  \label{eq:reduced-info-bottleneck}
\end{align}
Minimizing $\fancy{S}^*_B$ \prettyref{eq:reduced-info-bottleneck} thus
minimizes the sum of the entropy of internal states and the entropy
required to encode sensory states given internal states. In other
words, it strikes an optimal balance between model complexity
$H(\lambda)$, and model accuracy $H(S | \lambda)$. Friston further
shows that \prettyref{eq:reduced-info-bottleneck} directly derives
from the free energy principle, closing potential loopholes in the
theory. Please see \cite{Friston2016} for a comprehensive application
of the free energy principle to agents and learning.  
Note also that the bulk of the theory assumes the ergodic
hypothesis.

\section{Perception as General Intelligence}

Since we are chiefly interested in stochastic problems in the physical
world, we propose a straightforward informal definition of
intelligence:
\begin{definition}
  Intelligence measures the capability of a mechanism to solve
  prediction problems.
\end{definition}
Mechanism is any physical machine as usual, see \cite{Dowe2011} which
suggests likewise.  Therefore, a general formulation of Solomonoff
induction, operator induction, might serve as a model of general
intelligence, as well \cite{solomonoff-threekinds}. Recall that
operator induction can infer any physically plausible cpdf, thus its
approximation can solve any classical supervised machine learning
problem. The only slight issue with \prettyref{eq:legg-uai} might be
that it seems to exclude classical AI systems that are not agents,
e.g., expert systems, machine learning tools, knowledge representation
systems, search and planning algorithms, and so forth, which are
somewhat more naturally encompassed by our informal definition.

\subsection{Is operator induction adequate?}

A question naturally arises as to whether operator induction can
adequately solve every prediction problem we require in AI. There are
two strong objections to operator induction that we know of. It is
argued that in a dynamic environment, as in a physical environment, we
must use an active agent model so that we can account for changes in
the environment, as in the space-time embedded agent
\cite{orseau-spacetime} which also provides an agent-based
intelligence measure. This objection may be answered by the simple
solution that each decision of an active intelligent system may be
considered a separate induction problem. The second objection is that
the basic Solomonoff induction can only predict the next bit, but not
the expected cumulative reward, which its extensions can solve. We
counter this objection by stating that we can reduce an agent model to
a perception and action-planning problem as in OOPS-RL \cite{oops}. In
OOPS-RL, the perception module searches for the best world-model given
the history of sensory input and actions in allotted time using OOPS,
and the planning module searches for the best control program using
the world-model of the perception module to determine the action
sequence that maximizes cumulative reward likewise. OOPS has a
generalized Levin Search \cite{levin-universalsearch-eng} which may be
tweaked to solve either prediction or optimization problems. Hutter
has also observed that standard sequence induction does not readily
address optimization problems \cite{hutter-aixigentle}. However,
Solomonoff induction is still complete in the sense of Turing, and can
infer any computable cpdf; and when the extension to Solomonoff
induction is applied to sequence prediction, it does not yield a better
error bound, which seems like a conundrum. On the other hand, 
Levin Search with a proper
universal probability density function (pdf) of programs can be
modified to solve induction problems (sequence, set, operator, and
sequence prediction with arbitrary loss), inversion problems (computer
science problems in P and NP), and optimization problems
\cite{solomonoff-progress}. The planning module of OOPS-RL likewise
requires us to write such an optimization program. In that sense, AIXI
implies yet another variation of Levin Search for solving a particular
universal optimization problem, however, it also has the unique
advantage that formal transformations between AIXI problem and many
important problems including function minimization and strategic games
have been shown \cite{hutter-aixigentle}. Nevertheless, the discussion
in \cite{solomonoff-progress} is rather brief. Also see
\cite{alpcan14} for a discussion of universal optimization.
\begin{proposition}
  A discrete-time universal RL model may be reduced to operator
  induction.
\end{proposition}
More formally, the perceptual task of an RL agent would be inferring
from a history the cumulative rewards in the future, without loss of
generality. Let the chronology $C$ be a sequence of sensory, reward,
and action data
$C = [ (s_1,r_1,a_1), \\(s_2,r_2,a_2), \dots, (s_n,r_n,a_n) ]$ where
$C_i$ accesses $i$th element, and $C_{i:j}$ accesses the subsequence
$[ C_i, C_{i+1}, \dots, C_j]$.  Let $r_c$ be the cumulative reward
function where $r_c(C, i, j) = \sum_{k=i}^{k\leq j} r_k$. After
observing $(s_n,r_n,a_n)$, we construct dataset $D_c$ as follows. For
every unique $(i,j)$ pair such that $1 < i \leq j \leq n$, we
concatenate history tuples $C_{1:(i-1)}$, and we form a question
string that also includes the next action, $i$ and $j$,
$q = [ (s_1,r_1,a_1), (s_2, r_2,a_2),\ldots, (s_{(i-1)}, r_{(i-1)},
a_{(i-1)})], a_i, i, j$,
and an answer string which is the cumulative reward
$a = r_c(C, i, j)$. Solving the operator induction problem for this
dataset $D_C$ will yield a cpdf which predicts cumulative rewards in
the future. After that, choosing the next action is a simple matter of
maximizing $r(C_{1:n}, a_i, n+1, \lambda)$ where $\lambda$ is the
planning horizon. The reduction causes quadratic blow-up in the number
of data items. Our somewhat cumbersome reduction suggests that all of
the intelligence here comes from operator induction, surely an argmax
function, or a summation of rewards does not provide it, but rather it
builds constraints into the task. In other words, we interpret that
the intelligence in an agent model is provided by inductive inference,
rather than an additional application of decision theory.

\section{Physical Quantification of Intelligence}

\prettyref{def:legg-uai} corresponds to any kind of
reinforcement-learning or goal-following agent in AI literature quite
well, and can be adapted to solve other kinds of problems. The
unsupervised, active inference agent approach is proposed instead of
reinforcement learning approach in \cite{friston-rl}, and the authors
argue that they did not need to invoke the notion of reward, value or
utility. The authors in particular claim that they could solve the
mountain-car problem by the free-energy formulation of perception.  We
thus propose a perceptual intelligence measure.

%\subsection{A Physical Measure of Operator Induction}

\subsection{Universal measure of perception fitness}

%This section introduces a physical measure of operator induction which
%only depends on fundamental variables.

%\subsubsection{Universal measure of perception fitness.}

Note that operator induction is considered to be insufficient to
describe universal agents such as AIXI, because basic sequence
induction is inappropriate for modelling optimization problems
\cite{hutter-aixigentle}. However, a modified Levin search procedure
can solve such optimization problems as in finding an optimal control
program \cite{oops}.  In OOPS-RL, the perception module searches for
the best world-model given the history of sensory input and actions in
allotted time using OOPS, and the planning module searches for the
best control program using the world-model of the perception module to
determine the control program that maximizes cumulative reward
likewise. In this paper, we consider the perception module of such a
generic agent which must produce a world-model, given sensory input.

We can use the intelligence measure \prettyref{eq:legg-uai} in a
physical theory of intelligence, however it contains terms like
utility that do not have physical units (i.e., we would be preferring
a more reductive definition).  We therefore attempt to obtain such a
measure using the more benign goodness-of-fit
(\prettyref{eq:opind-gof}). Let the universal measure of the fitness
of operator induction be defined as
\begin{equation}
  \label{eq:soll-uai}
  \Upsilon_O(\pi) = \sum_{\mu \in S} 2^{-H_U(\mu)}\Psi(\mu, \pi)
\end{equation}
where $S$ is the set of possible stochastic sources in the observable
universe $U$ and $\pi$ is a physical mechanism, and $\Psi$ is relative
to a stochastic source $\mu$ and a physical mechanism (computer)
$\pi$.  This would be maximum if we assume that operator induction
were solved exactly by an oracle machine.

Note that $H_U(\mu)$ is finite; $\Psi(\mu, \pi)$ is likewise bounded
by the amount of computation $\pi$ will spend on approximating
operator induction.

\subsection{Application to homeostasis agent}
In a presentation to Friston's group in January 2015, we noted that
the minimization of $\fancy{S}^*_B$ is identical to Minimum Message
Length principle, which can be further refined as
\begin{align}
  \fancy{S}'_B = H^*(\Lambda) + H^*(S | \Lambda)
  \label{eq:sol-info-bottleneck}
\end{align}
using Solomonoff's entropy formulation that takes the negative
logarithm of algorithmic probability \cite{solcomplexity}. In the
unsupervised agent context, solving this minimization problem
corresponds to inferring an optimal behavioral policy as $\Lambda$
constitutes internal dynamics which may be modeled as a
non-terminating program. We could directly apply induction to minimize
KL divergence, as well. Note the correspondence to operator induction.
\begin{theorem}
  Minimizing the free energy is equivalent to solving the operator
  induction problem for $(\lambda, s)$ pairs where $q_i \in \Lambda$
  and $a_i \in S$.
\end{theorem}
\begin{proof}
  Observe that minimizing \prettyref{eq:reduced-info-bottleneck}
  corresponds to picking maximum $\psi^j_n$ since in entropy form,
  \begin{gather*}
    -\log_2(\psi^j_n)= -\log_2(2^{-|O^j(\cdot|\cdot)|}) -\log_2(\prod_{i=1}^n{O^j(s_i|\lambda_i)})\\
    = |O^j(\cdot|\cdot)| - \sum_{i=1}^n \log_2(O^j(a_i|q_i)) =
    |O^j(\cdot|\cdot)| + H(O^j(a_i|q_i)) .
  \end{gather*}
  We define a non-redundant selection of $\psi^j_n$'s,
  $|O^j(\cdot|\cdot)|=H_U(O^j(\cdot|\cdot))$, e.g., we pick only the
  shortest programs that produce the same cpdf, otherwise the entropy
  form would diverge.  Minimizing \prettyref{eq:sol-info-bottleneck}
  is \textit{exactly} operator induction, even though the questions
  are programs, the ensemble here is of all programs and all sensory
  state, program pairs in space-time.
  $\sum |O^j(\cdot|\cdot)| = H^*(\Lambda)$ and
  $\sum H(O^j(a_i|q_i)) = H^*(S|\Lambda)$.  Note that this merely
  establishes model equivalence, we have not yet explained how it is
  to be computed in detail.
\end{proof}
\begin{proposition}
  By the above theorem, \prettyref{eq:soll-uai} measures the goodness
  of fit for a given homeostasis agent mechanism, for all possible
  environments.
\end{proposition}
% \begin{proof}
The mechanism $\pi$ that maximizes $\Psi(\mu,\pi)$ achieves less error
with respect to a source (which may be taken to correspond to the
whole random dynamical system in the framework of free energy
principle), while $\Upsilon_O(\pi)$ normalizes $\Psi(\mu,\pi)$ with
respect to a random dynamical system. It holds for the same reasons
Legg's measure holds, which are not discussed due to space limits in
the present paper.
% \end{proof}
We prefer the unsupervised homeostasis agent among the two agent
models we discussed because it provides an
exceptionally elegant and reductionist model of autonomous behavior,
that has been rigorously formulated physically. Note that this agent
is conceptually related to the survival property of RL agents
discussed in \cite{ring2011delusion}.

\subsection{Discussion} 
The unsupervised model still achieves exploration and curiosity,
because it would stochastically sample and navigate the environment to
reduce predictive errors.  While we either optimize perceptual models
or choose an action that would befit expectations, it might be
possible to express the optimal adaptive agent policy in a general
optimization framework.  A more in-depth analysis of the unsupervised
agent will be presented in a subsequent publication. A more general
reductive definition of intelligence should also be researched.  These
developments could eventually help unify AGI theory.

\bibliographystyle{splncs03} \bibliography{agi,physics,complexity}

\begin{thebibliography}{10}
\providecommand{\url}[1]{\texttt{#1}}
\providecommand{\urlprefix}{URL }

\bibitem{alpcan14}
Alpcan, T., Everitt, T., Hutter, M.: Can we measure the difficulty of an
  optimization problem? In: 2014 {IEEE} Information Theory Workshop, {ITW}
  2014, Hobart, Tasmania, Australia, November 2-5, 2014. pp. 356--360. {IEEE}
  (2014), \url{http://dx.doi.org/10.1109/ITW.2014.6970853}

\bibitem{ashby1962principles}
Ashby, W.R.: Principles of the self-organizing system. In: v.~Foerster, H.,
  Zopf, G.W. (eds.) Principles of Self-Organization: Transactions of the
  University of Illinois Symposium, pp. 255--278. Pergamon, London (1962)

\bibitem{ashby1947principles}
Ashby, W.: {Principles of the self-organizing dynamic system.} The Journal of
  General Psychology  37(2),  125--128 (1947)

\bibitem{Dowe2011}
Dowe, D.L., Hern{\'a}ndez-Orallo, J., Das, P.K.: Artificial General
  Intelligence: 4th International Conference, AGI 2011, Mountain View, CA, USA,
  August 3-6, 2011. Proceedings, chap. Compression and Intelligence: Social
  Environments and Communication, pp. 204--211. Springer Berlin Heidelberg,
  Berlin, Heidelberg (2011),
  \url{http://dx.doi.org/10.1007/978-3-642-22887-2_21}

\bibitem{Friston2016}
Friston, K., FitzGerald, T., Rigoli, F., Schwartenbeck, P., O⿿Doherty, J.,
  Pezzulo, G.: Active inference and learning. Neuroscience and Biobehavioral
  Reviews  68,  862 -- 879 (2016),
  \url{http://www.sciencedirect.com/science/article/pii/S0149763416301336}

\bibitem{Friston2006}
Friston, K., Kilner, J., Harrison, L.: A free energy principle for the brain.
  Journal of Physiology-Paris  100(1–3),  70 -- 87 (2006),
  \url{http://www.sciencedirect.com/science/article/pii/S092842570600060X},
  theoretical and Computational Neuroscience: Understanding Brain Functions

\bibitem{friston-rl}
Friston, K.J., Daunizeau, J., Kiebel, S.J.: Reinforcement learning or active
  inference? PLOS ONE  4(7),  1--13 (07 2009),
  \url{https://doi.org/10.1371/journal.pone.0006421}

\bibitem{hutter-aixigentle}
Hutter, M.: Universal algorithmic intelligence: A mathematical
  top$\rightarrow$down approach. In: Goertzel, B., Pennachin, C. (eds.)
  Artificial General Intelligence, pp. 227--290. Cognitive Technologies,
  Springer, Berlin (2007)

\bibitem{friston-free-energy}
Karl, F.: A free energy principle for biological systems. Entropy  14(11),
  2100--2121 (2012), \url{http://www.mdpi.com/1099-4300/14/11/2100}

\bibitem{legg-uai}
Legg, S., Hutter, M.: Universal intelligence: A definition of machine
  intelligence. Minds Mach.  17(4),  391--444 (Dec 2007)

\bibitem{legg-aiq}
Legg, S., Veness, J.: An approximation of the universal intelligence measure.
  In: Algorithmic Probability and Friends. Bayesian Prediction and Artificial
  Intelligence, Lecture Notes in Computer Science, vol. 7070, pp. 236--249.
  Springer Berlin Heidelberg (2013)

\bibitem{levin-universalsearch-eng}
Levin, L.: {Universal problems of full search.} Problems of Information
  Transmission  9(3),  256--266 (1973)

\bibitem{levin-thesis}
Levin, L.A.: Some theorems on the algorithmic approach to probability theory
  and information theory. CoRR  abs/1009.5894 (2010)

\bibitem{Li2008}
Li, M., Vitanyi, P.M.: An Introduction to Kolmogorov Complexity and Its
  Applications. Springer Publishing Company, Incorporated, 3 edn. (2008)

\bibitem{seth-ultimate}
{Lloyd}, S.: {Ultimate physical limits to computation}. \nat  406 (Aug 2000)

\bibitem{orseau-spacetime}
Orseau, L., Ring, M.: Space-time embedded intelligence. In: Bach, J., Goertzel,
  B., Iklé, M. (eds.) Artificial General Intelligence, Lecture Notes in
  Computer Science, vol. 7716, pp. 209--218. Springer Berlin Heidelberg (2012),
  \url{http://dx.doi.org/10.1007/978-3-642-35506-6_22}

\bibitem{ozkuraluip2arxiv}
{{\"O}zkural}, E.: {Ultimate Intelligence Part II: Physical Measure and
  Complexity of Intelligence}. ArXiv e-prints  (Apr 2015)

\bibitem{ozkural-agi15}
{\"{O}}zkural, E.: Ultimate intelligence part {I:} physical completeness and
  objectivity of induction. In: Artificial General Intelligence - 8th
  International Conference, {AGI} 2015, {AGI} 2015, Berlin, Germany, July
  22-25, 2015, Proceedings. pp. 131--141 (2015),
  \url{http://dx.doi.org/10.1007/978-3-319-21365-1_14}

\bibitem{ring2011delusion}
Ring, M., Orseau, L.: Delusion, survival, and intelligent agents. In:
  Artificial General Intelligence, pp. 11--20. Springer Berlin Heidelberg
  (2011)

\bibitem{oops}
Schmidhuber, J.: Optimal ordered problem solver. Machine Learning  54,
  211--256 (2004)

\bibitem{alp1}
Solomonoff, R.J.: A formal theory of inductive inference, part i. Information
  and Control  7(1),  1--22 (March 1964)

\bibitem{solcomplexity}
Solomonoff, R.J.: Complexity-based induction systems: Comparisons and
  convergence theorems. IEEE Trans. on Information Theory  IT-24(4),  422--432
  (July 1978)

\bibitem{solomonoff-progress}
Solomonoff, R.J.: Progress in incremental machine learning. Tech. Rep.
  IDSIA-16-03, IDSIA, Lugano, Switzerland (2003)

\bibitem{solomonoff-threekinds}
Solomonoff, R.J.: Three kinds of probabilistic induction: Universal
  distributions and convergence theorems. The Computer Journal  51(5),
  566--570 (2008)

\bibitem{tishby-infobottleneck}
{Tishby}, N., {Pereira}, F.C., {Bialek}, W.: {The information bottleneck
  method}. ArXiv Physics e-prints  (Apr 2000)

\bibitem{wallace99}
Wallace, C.S., Dowe, D.L.: Minimum message length and kolmogorov complexity.
  The Computer Journal  42(4),  270--283 (1999),
  \url{http://comjnl.oxfordjournals.org/content/42/4/270.abstract}

\bibitem{WallaceBoulton:1968}
Wallace, C.S., Boulton, D.M.: A information measure for classification.
  Computer Journal  11(2),  185--194 (1968)

\end{thebibliography}

\end{document}